\documentclass[10pt,twocolumn,letterpaper]{article}

\usepackage{wacv}
\usepackage{times}
\usepackage{epsfig}
\usepackage{graphicx}
\usepackage{amsmath, bm}
\usepackage{amssymb}

\usepackage{hyperref}
\usepackage{url}
\usepackage{booktabs}
\usepackage{amsfonts}
\usepackage{nicefrac}
\usepackage{microtype}
\hypersetup{colorlinks=true, linkcolor=red}
\usepackage{pgfplots}
\usepackage{pgfplotstable}
\usepackage{filecontents,pgfplots}
\usepackage{tikz}
\usetikzlibrary{arrows,calc,shapes,snakes,positioning,matrix,arrows,decorations.pathmorphing,decorations.text}
\usepackage{mathrsfs}
\usepackage{soul}
\usepackage{multirow}
\usepackage{comment}
\usepackage{rotating}
\usepackage{mwe}
\PassOptionsToPackage{bookmarks=false}{hyperref}

\usepackage[skip=1ex]{caption}
\usepackage{array, booktabs, makecell, multirow}
\setcellgapes{5pt}
\usepackage{hyperref}

\newcommand{\RR}{\mathcal{R}}

\usepackage{color}
\usepackage[utf8]{inputenc} 
\usepackage[T1]{fontenc}    
\usepackage{booktabs}       
\usepackage{amsfonts}       
\usepackage{nicefrac}       
\usepackage{microtype}      

\usepackage{pgfplots}
\usepackage{pgfplotstable}
\usepackage{filecontents,pgfplots}
\usepackage{tikz}
\usetikzlibrary{arrows,calc,shapes,snakes,positioning,matrix,arrows,decorations.pathmorphing,decorations.text}
\usepackage{sci}

\usepackage{url}
\usepackage{graphicx}
\usepackage{tikz}
\usepackage{pgfplots}

\usepackage{sci}

\usepackage{amsmath,amssymb} 
\usepackage{color}
\usepackage{amsmath}
\usepackage{amssymb}
\usepackage{amsthm}
\usepackage{algorithm}
\usepackage{algorithmicx}
\usepackage{algpseudocode}
\usepackage{algpascal}

\usepackage{mathrsfs}
\usepackage{soul}
\usepackage{multirow}
\usepackage{comment}
\usepackage{xfrac}
\usepackage{rotating}
\usepackage{hhline}
\usepackage{mathtools}
\usepackage{wrapfig}
\usepackage{upgreek}
\usepackage{tikz}
\usepackage{caption}
\usepackage{booktabs}
\usepackage{subfigure}
\usepackage{lipsum}
\hypersetup{breaklinks=true}
\usepackage{dblfloatfix}

\newcommand{\tfe}{\boldsymbol{\uptheta}_{fe}}
\newcommand{\tc}{\boldsymbol{\uptheta}_{c}}
\newcommand{\tbp}{\boldsymbol{\uptheta}_{bp}}

\newcommand{\FE}{\mathbb{F}\mathbb{E}}
\newcommand{\CC}{\mathbb{C}}
\newcommand{\BP}{\mathbb{B}\mathbb{P}}

\newtheorem*{remark}{Remark}

\def\wacvPaperID{1345} 

\wacvfinalcopy 

\ifwacvfinal
\def\assignedStartPage{9876} 
\fi
\ifwacvfinal
\else
\fi

\begin{document}

\title{Representation Learning with Statistical Independence to Mitigate Bias}

\author{Ehsan Adeli$^{1,2}$\thanks{Equal Contribution}, Qingyu Zhao$^{1*}$, Adolf Pfefferbaum$^{1,3}$, Edith V. Sullivan$^{1}$, \\
Li Fei-Fei$^{2}$, Juan Carlos Niebles$^{2}$, Kilian M. Pohl$^{1,3}$\\
$^1$Department of Psychiatry and Behavioral Sciences, Stanford University, CA 94305 \\
  $^2$Department of Computer Science, Stanford University, CA 94305\\
  $^3$Center for Biomedical Sciences, SRI International, Menlo Park, CA 94205\\
{\tt\small \{eadeli,\,qingyuz,\,dolfp,\,edie,\,feifeili,\,jniebles,\,kilian.pohl\}@stanford.edu}
}


\maketitle

\begin{abstract}
Presence of bias (in datasets or tasks) is inarguably one of the most critical challenges in machine learning applications that has alluded to pivotal debates in recent years. Such challenges range from spurious associations between variables in medical studies to the bias of race in gender or face recognition systems. Controlling for all types of biases in the dataset curation stage is cumbersome and sometimes impossible. The alternative is to use the available data and build models incorporating fair representation learning. In this paper, we propose such a model based on adversarial training with two competing objectives to learn features that have (1) maximum discriminative power with respect to the task and (2) minimal statistical mean dependence with the protected (bias) variable(s). Our approach does so by incorporating a new adversarial loss function that encourages a vanished correlation between the bias and the learned features. We apply our method to synthetic data, medical images (containing task bias), and a dataset for gender classification (containing dataset bias). Our results show that the learned features by our method not only result in superior prediction performance but also are unbiased. The code is available at {\small{\url{https://github.com/QingyuZhao/BR-Net/}}}. 
\end{abstract}

\section{Introduction}

A central challenge in practically all machine learning applications is how to identify and mitigate the effects of the bias present in the study. Bias can be defined as one or a set of extraneous protected variables that distort the relationship between the input (independent) and output (dependent) variables and hence lead to erroneous conclusions \cite{pourhoseingholi2012control}. In a variety of applications ranging from disease prediction to face recognition, machine learning models are built to predict labels from images. Variables such as age, sex, and race may influence the training if the labels distribution is skewed with respect to them. Hence, the model may learn bias effects instead of actual discriminative cues.

The two most prevalent types of biases are dataset bias \cite{salimi2019interventional,khosla2012undoing} and task bias \cite{lee2019algorithmic,huang2012biased}. \textit{Dataset bias} is often introduced due to the lack of enough data points spanning the whole spectrum of variations with respect to one or a set of protected variables (\ie, variables that define the bias). For example, a model that predicts gender from face images may have different recognition capabilities for different races with uneven sizes of training samples \cite{buolamwini2018gender}. 
\textit{Task bias}, on the other hand, is introduced by the intrinsic dependency between protected variables and the task. For instance, in neuroimaging applications, demographic variables such as gender \cite{eichler1992gender} or age \cite{dobrowolska2019age} are crucial protected variables; \ie,  they affect both the input (\eg, neuroimages) and output (\eg, diagnosis) of a prediction model so they likely introduce a distorted association. Both bias types pose serious challenges to learning algorithms.

With the rapid development of deep learning methods, Convolutional Neural Networks (CNN) are emerging as eminent ways of extracting representations (features) from imaging data. However, like other machine learning methods, CNNs are prone to capturing any bias present in the task or dataset when not properly controlled. Recent work has focused on methods for understanding causal effects of bias on databases  \cite{salimi2019interventional,khademi2019algorithmic} or learning fair models \cite{li2019repair,tommasi2017deeper,khosla2012undoing,salimi2019interventional} with de-biased representations based on the developments in invariant feature learning \cite{ganin2016domain,bechavod2017learning} and domain adversarial learning \cite{elazargoldberg2018adversarial,sadeghi2019global}. These methods have shown great potential when the protected variables are dichotomous or categorical. However, their applications to handling task bias and continuous protected variables are still under-explored.

\begin{figure*}[t]
    \centering
    \includegraphics[width=0.98\textwidth]{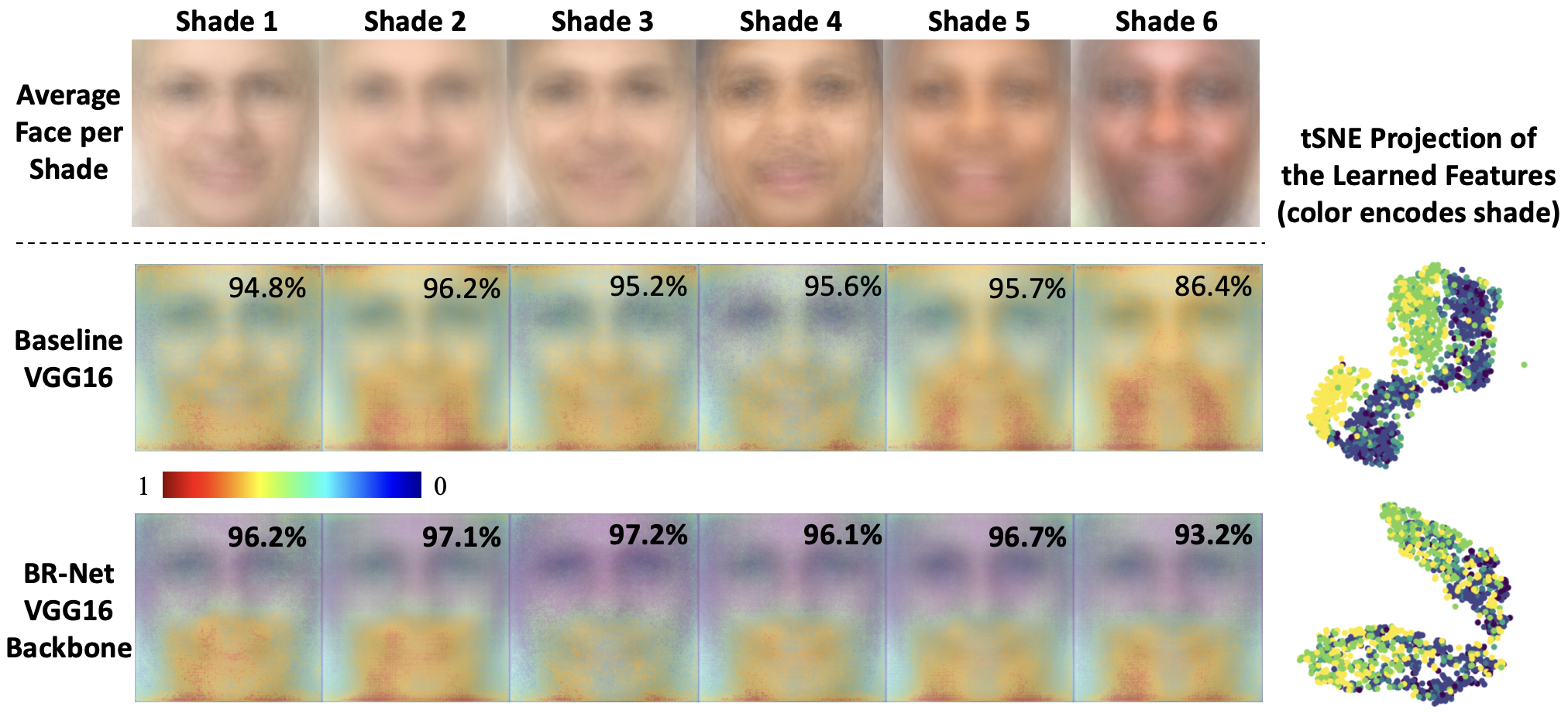}\vspace{-5pt}
    \caption{Average face images for each shade category ($1^{st}$ row), average saliency map of the trained baseline ($2^{nd}$ row), and BR-Net ($3^{rd}$ row) color-coded with the normalized saliency for each pixel. BR-Net results in more stable patterns across all 6 shades. The last column shows the tSNE projection of the learned representations by each method. Our method results in a better representation space invariant to the bias variable (shade) while the baseline shows a pattern influenced by the bias. Average accuracy of per-shade gender classification over 5 runs of 5-fold cross-validation (pre-trained on ImageNet, fine-tuned on GS-PPB) is shown on each average map. 
    BR-Net not only obtains better accuracy for the darker shade but also \textit{regularizes} the model to improve overall per-category accuracy. }
    \label{fig:gs_viz}
\end{figure*}


In this paper, we propose a representation learning scheme that learns features predictive of class labels with minimal bias to any generic type of protected variables. Our method is inspired by the domain-adversarial training approaches \cite{ganin2016domain} with controllable invariance \cite{xie2017controllable} within the context of generative adversarial networks (GANs) \cite{goodfellow2014generative}. We introduce an adversarial loss function based on the Pearson correlation between true and predicted values of a protected variable. Unlike prior methods, this strategy can handle protected variables that are continuous or ordinal. We theoretically show that the adversarial minimization of the linear correlation can remove non-linear association between the learned representations and protected variables, thus achieving \textit{statistical mean independence}. Further, our strategy improves over the commonly used cross-entropy or mean-squared error (MSE) loss that is often \textit{ill-posed} when optimized adversarially. Our method, denoted by Bias-Resilient Neural Network (BR-Net), uses architectures similar to the prior adversarial invariant feature learning works, but injects resilience towards the bias during training by taking an statistical independence approach to produce bias-invariant features at the presence of dataset and task biases. BR-Net is novel compared to the prior fair representation learning methods as (1) it can deal with \textit{continuous} and \textit{ordinal} protected variables and (2) is based on a  theoretical proof of mean independence within the adversarial training context.

We evaluate BR-Net on two datasets that allow us to highlight different aspects of the method in comparison with a wide range of baselines. First, we test it on a \textit{medical imaging application}, \ie, distinguishing T1-weighted Magnetic Resonance Images (MRIs) of patients with the human immunodeficiency virus (HIV) from those of healthy subjects. As documented by the HIV literature, HIV accelerates the aging process of the brain \cite{cole2017increased}, thereby introducing a task bias with respect to age (a continuous variable). In other words, if a predictor is trained not considering age as a protected variable (or confounder as referred to in medical studies), the predictor may actually learn the brain aging patterns rather than actual HIV markers. Then, we evaluate BR-Net for \textit{gender classification} using the Gender Shades Pilot Parliaments Benchmark (GS-PPB) dataset \cite{buolamwini2018gender}. We use different backbones pre-trained on ImageNet \cite{deng2009imagenet} in BR-Net and fine-tune them for our specific task, \ie, gender prediction from face images. We show that prediction accuracy of the vanilla model is dependent on the subject's skin color (quantified by the `shade' variable, an ordinal variable), which is not the case for BR-Net. Our comparison with several baselines  and prior state-of-the-art 
shows that BR-Net not only learns features impartial to race (verified by feature embedding visualizations) but also results in higher accuracy (Fig. \ref{fig:gs_viz}). 

\section{Related Work}
\noindent\textbf{Fairness in Machine Learning:} In recent years, developing fair machine learning models have been the center of many discussions \cite{liu2018delayed,hashimoto2018fairness,barocas2017fairness} including the media \cite{Dhruv2019,Miller2015}. It is often argued that human or society biases are replicated in the training datasets and hence can be seen in learned models \cite{barocas2016big}. Recent effort in solving this problem focused on building fairer datasets \cite{yang2019towards,celis2016fair,salimi2019interventional}. However, this approach is not always practical for large-scale datasets or for applications, where data is relatively scarce and expensive to generate (\eg, medical applications). Other works learn fair representations leveraging the existing data \cite{zemel2013learning,creager2019flexibly,weinzaepfel2019mimetics} by identifying features that are only predictive of the actual outputs, \ie, impartial to the protected variable. But they come short when the protected variables are continuous. 

\vspace{5pt}\noindent\textbf{Domain-Adversarial Training:} \cite{ganin2016domain} proposed for the first time to use adversarial training for domain adaptation tasks by using the learned features to predict the domain label (a binary variable; source or target). Several other works built on top of the same idea and explored different loss functions \cite{bousmalis2017unsupervised}, domain discriminator settings  \cite{tzeng2017adversarial,Edwards16,beutel2017data}, or cycle-consistency \cite{hoffman2017cycada}. The focus of all these works was to close the domain gap, which is often encoded as a binary variable. To learn general-purpose bias-resilient models, we need new theoretical insight into the methods that can learn features invariant to all types of protected variables. 

\vspace{5pt}\noindent\textbf{Invariant Representation Learning:} There have been different attempts in the literature for learning representations that are invariant to specific factors in the data. For instance, \cite{zemel2013learning} took an information obfuscation approach to obfuscate membership in the protected group of data during training, and \cite{bechavod2017learning,ranzato2007unsupervised} introduced regularization-based methods. Recently, \cite{xie2017controllable,akuzawa2019adversarial,zhang2018mitigating,elazargoldberg2018adversarial,choi2019can} proposed to use domain-adversarial training strategies for invariant feature learning. Some works \cite{sadeghi2019global,wang2019balanced} used adversarial techniques based on similar loss functions as in domain adaptation to predict the exact values of the protected variables. For instance, \cite{wang2019balanced} used a binary cross-entropy for removing effect of `gender' and \cite{sadeghi2019global} used linear and kernelized least-square predictors as the adversarial component. Several methods based on optimizing equalized odds \cite{madras2018learning}, entropy \cite{springenberg2015unsupervised,roy2019} and mutual-information  \cite{song2019,bertran19a,moyer2018} were also widely used for fair representation learning. However, these methods are intractable for continuous or ordinal protected variables. 



\section{Bias-Resilient Neural Network (BR-Net)}
Suppose we have an $M$-class classification problem, for which we have $N$ pairs of training images and their corresponding target label(s): $\{(\X_i, \y_i)\}_{i=1}^N$. Assuming a set of $k$ protected variables, denoted by a vector $\b \in \RR^k$, to train a model for classifying each image while being impartial to the protected variables, we propose an \textit{end-to-end} architecture (Fig.~\ref{fig:arch}) similar to domain-adversarial training approaches \cite{ganin2016domain}. Given the input image $\X$, the representation learning ($\FE$) module extracts a feature vector $\F$, on top of which a Classifier ($\CC$) predicts the class label $\y$. 
Now, to guarantee that these features are not biased to $\b$, 
we build a network (denoted by $\BP$) with a new loss function that checks the statistical mean dependence of the protected variables to $\F$. Back-propagating this loss to $\FE$ adversarially results in features that minimize the classification loss while having the least statistical dependence on the protected variables. 

\begin{figure}[t]
    \centering
    \includegraphics[width=\linewidth]{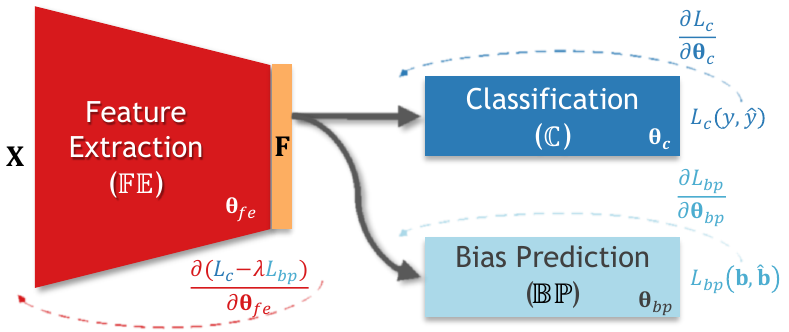} 
    \caption{BR-Net architecture: 
    $\FE$ learns features, $\F$, that successfully classify ($\CC$) the input while being invariant (statistically independent) to the protected variables, $\b$, using $\BP$ and the adversarial loss, $-\lambda L_{bp}$ (based on correlation coefficient). Forward arrows show forward paths while the backward dashed ones indicate back-propagation with the respective gradient ($\partial$) values.}
    \label{fig:arch}
\end{figure}
\begin{figure}[t]
    \centering
    \includegraphics[width=\linewidth]{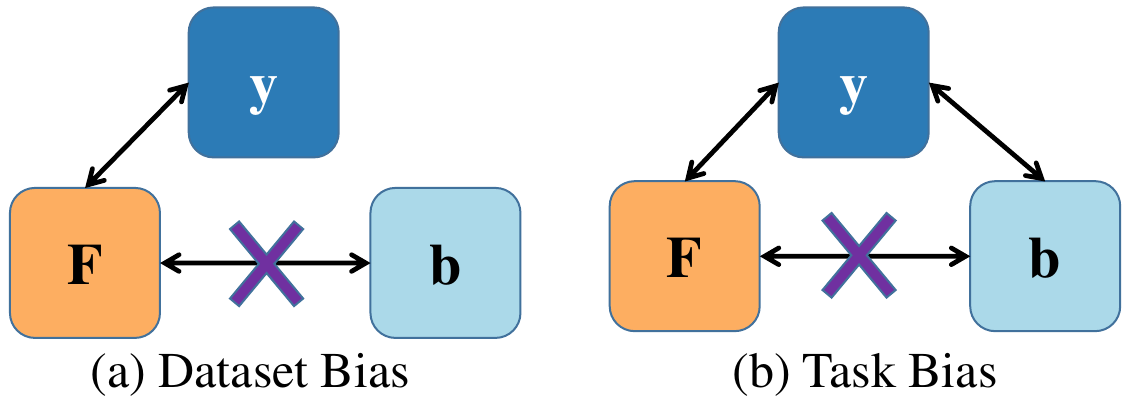}
    \caption{BR-Net can remove direct dependency between $\F$ and $\b$ for both dataset or task bias. }
    \label{fig:arch2}
\end{figure}

Each network has its underlying trainable parameters, defined as $\tfe$ for $\FE$, $\tc$ for $\CC$, and $\tbp$ for $\BP$.  If the predicted probability that subject $i$ belongs to class $m$ is defined by $\hat{y}_{im} = \CC(\FE(\X_i; \tfe); \tc)$, the classification loss can be characterized by a cross-entropy:
\begin{equation}
    L_{c} (\X, \y; \tfe, \tc) 
    = - \sum_{i=1}^N \sum_{m=1}^M y_{im} \log(\hat{y}_{im}).
\end{equation}
Similarly, with $\hat{\b}_i=\BP(\FE(\X_i; \tfe); \tbp)$, we can define the adversarial component of the loss function. Standard methods for designing this loss function suggest to use a cross-entropy for binary/categorical variables (\eg, in \cite{ganin2016domain,xie2017controllable,choi2019can}) or an $\ell_2$ MSE loss for continuous variables (\cite{sadeghi2019global}). These loss functions solely aim to maximize prediction error of $\b$ in the adversarial training but cannot remove statistical dependence. For example, the maximization of MSE between $\hat{\b}$ and $\b$ is an ill-posed (unbounded) objective and can be trivially achieved by uniformly shifting the magnitude of $\hat{\b}$, which does not remove any correlation with respect to $\b$. To avoid this issue, we define the surrogate loss for predicting the protected variables while quantifying the statistical dependence with respect to $\b$ based on the 
squared Pearson correlation $\text{corr}^2(\cdot,\cdot)$:
\begin{equation} \label{eq:corr}
    L_{bp} (\X, \b; \tfe, \tbp) = -\sum_{\kappa=1}^{k}  \text{corr}^2(\b_\kappa, \hat{\b}_\kappa),
\end{equation}
where $\b_\kappa$ defines the vector of $\kappa^\text{th}$ protected variable across all $N$ training inputs. Through adversarial training, we aim to remove statistical dependence by encouraging a zero correlation between $\b_\kappa$ and $\hat{\b}_\kappa$. Note, $\BP$ deems to maximize squared correlation and $\FE$ minimizes for it. Since $\text{corr}^2$ is bounded in the range [0, 1], both minimization and maximization schemes are feasible. 
Hence, the overall objective of the network is then defined as
\begin{equation}
    \min_{\tfe, \tc} \max_{\tbp} ~ L_c (\X, \y; \tfe, \tc) \\
     - \lambda L_{bp} (\X, \b; \tfe, \tbp).
\end{equation}
where hyperparameter $\lambda$ controls the trade-off between the two objectives. This scheme is similar to GAN \cite{goodfellow2014generative} and domain-adversarial training \cite{ganin2016domain,xie2017controllable}, in which a min-max game is defined between two networks. In our case, $\FE$ extracts features that minimize the classification criterion, while `fooling' $\BP$ (\ie, making $\BP$ incapable of predicting the protected variables). Hence, the saddle point for this objective is obtained when the parameters $\tfe$ minimize the classification loss while maximizing the loss of $\BP$. Similar to the training of GANs, in each iteration, we first back-propagate the $L_c$ loss to update $\tfe$ and $\tc$. With $\tfe$ fixed, we then minimize the $L_{bp}$ loss to update $\tbp$. Finally, with $\tbp$ fixed, we maximize the $L_{bp}$ loss to update $\tfe$. 
In this study, $L_{bp}$ depends on the correlation operation, which is a population-based operation, as opposed to individual-level error metrics such as cross-entropy or MSE losses. Therefore, we calculate the correlations over each training batch as a batch-level operation. 

\subsection{Non-linear Statistical Independence Guarantee} 
In general, a zero-correlation or a zero-covariance only quantifies linear independence between univariate variables but cannot infer non-linear relationships in high dimension. However, we now theoretically show that, under certain assumptions on the adversarial training of $\BP$, a zero-covariance would guarantee the \textit{mean independence} \cite{wooldridge2010} between protected variables and the high dimensional features, a much stronger type of statistical independence than the linear one.

A random variable $\mathcal{B}$ is said to be \textit{mean independent} of $\mathcal{F}$ if and only if $E[\mathcal{B} |\mathcal{F}=\xi]=E[\mathcal{B}]$ for all $\xi$ with non-zero probability, where $E[\cdot]$ defines the expected value. In other words, the expected value of $\mathcal{B}$ is neither linearly nor non-linearly dependent on $\mathcal{F}$, but the variance of $\mathcal{B}$ might. The following theorem then relates the mean independence between features $\mathcal{F}$ and the protected variables $\mathcal{B}$ to the zero-covariance between $\mathcal{B}$ and the $\BP$ prediction, $\hat{\mathcal{B}}$. 

\newtheorem{theorem}{Theorem}

\noindent\textit{Property 1: $\mathcal{B}$ is mean independent of {\small $\hat{\mathcal{B}}\Rightarrow Cov(\mathcal{B},\hat{\mathcal{B}})=0$}. }

\noindent\textit{Property 2: $\mathcal{B},\mathcal{F}$ are mean independent $\Rightarrow$ $\mathcal{B}$ is mean independent of $\hat{\mathcal{B}}=\phi(\mathcal{F})$ for any mapping function $\phi$. }

\begin{theorem}
Given random variables $\mathcal{F}, \mathcal{B}, \hat{\mathcal{B}}$ with finite second moment, $\mathcal{B}$ is mean independent of $\mathcal{F}$ $\Leftrightarrow$ for any arbitrary mapping $\phi$, s.t. $\hat{\mathcal{B}}=\phi(\mathcal{F})$, $cov(\mathcal{B},\hat{\mathcal{B}})=0$
\end{theorem}

\begin{proof}
The forward direction $\Rightarrow$ follows directly through \textit{Property} 1 and 2. We focus the proof on the reverse direction. Now, construct a mapping function $\hat{\mathcal{B}}=\phi(\mathcal{F})=E[\mathcal{B}|\mathcal{F}]$, \ie, $\phi(\xi)=E[\mathcal{B}|\mathcal{F}=\xi]$, then $Cov(\mathcal{B},\hat{\mathcal{B}}) = 0$ implies
\begin{equation}
    E\big[\mathcal{B} E[\mathcal{B}|\mathcal{F}]\big]=E[\mathcal{B}]E\big[E[\mathcal{B}|\mathcal{F}]\big].
    \label{eq:covariance}
\end{equation}
Due to the self-adjointness of the mapping $\mathcal{B}$ $\mapsto$ $E[\mathcal{B}|\mathcal{F}]$, the left hand side of Eq. (\ref{eq:covariance}) reads $E\big[\mathcal{B} E[\mathcal{B}|\mathcal{F}]\big]=E\big[\left(E[\mathcal{B}|\mathcal{F}]\right)^2\big]=E[\hat{\mathcal{B}}^2]$. By the law of total expectation $E\big[E[\mathcal{B}|\mathcal{F}]\big]=E[\mathcal{B}]$, the right hand side of Eq. (\ref{eq:covariance}) becomes $E[\hat{\mathcal{B}}]^2$. By Jensen's (in)equality, $E[\hat{\mathcal{B}}^2]=E[\hat{\mathcal{B}}]^2$ holds iff 
$\hat{\mathcal{B}}$ is a constant, \ie, 
$\mathcal{B}$ is mean independent of $\mathcal{F}$. 
\end{proof}
\begin{remark}
In practice, we normalize the covariance by standard deviations of variables for optimization stability. In the unlikely singular case that $\BP$ outputs a constant, we add a small perturbation in computing the standard deviation.
\end{remark}
This theorem echoes the validity of our adversarial training strategy: $\FE$ encourages a zero-correlation between $\b_\kappa$ and $\hat{\b}_\kappa$, which enforces $\b_\kappa$ to be mean independent of $\F$ (one cannot infer the expected value of $\b_\kappa$ from $\F$). In turn, assuming $\BP$ has the capacity to approximate any arbitrary mapping function, the mean independence between features and bias would correspond to a zero-correlation between $\b_\kappa$ and $\hat{\b}_\kappa$, otherwise $\BP$ would adversarially optimize for a mapping function that increases the correlation. 

Moreover, \textbf{Theorem 1} induces that when $\b_\kappa$ is mean independent of $\F$, $\b_\kappa$ is also mean independent of $\y$ for any arbitrary classifier $\CC$, indicating that the prediction is guaranteed to be unbiased. When $\CC$ is a binary classifier and $\y \sim \mbox{Ber}(q)$, we have $p(\y=1|\b_\kappa)=E[\y|\b_\kappa]=E[\y]=p(\y=1)=q$; that is, $\y$ and $\b_\kappa$ are fully independent. 

\begin{figure*}[!t]
    \centering
    \subfigure[Synthetic data image synthesis.]{ \includegraphics[width=0.44\textwidth,trim=0 -20 0 0, clip]{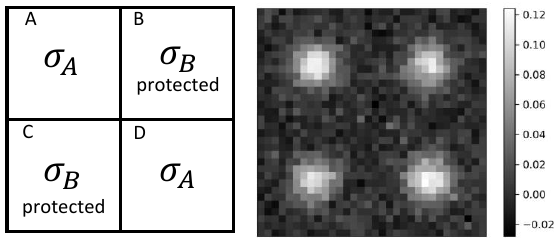}} ~~~~~
    \subfigure[Distance correlation w.r.t. $\sigma_B$]{\includegraphics[width=0.34\textwidth,trim=5 0 8 5,clip]{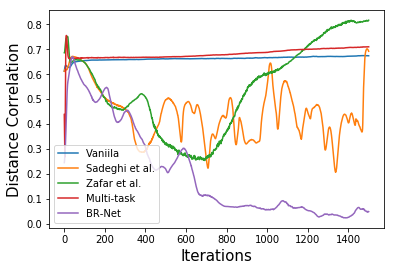}}
    \caption{Formation of synthetic dataset (a) and comparison of results for different methods (b). }
    \label{fig:synthetic_setup}
\end{figure*}

As mentioned, when $\b$ characterizes dataset bias (Fig. \ref{fig:arch}a), there is no intrinsic link between the protected variable and the task label (\eg, in gender recognition, probability of being a female is not dependent on race), and the bias is introduced due to the data having a skewed distribution with respect to the protected variable. In this situation, we should train $\BP$ on the entire dataset to remove dependency between $\F$ and $\b$. On the other hand, when $\b$ is a task bias (Fig. \ref{fig:arch}b), it will have an intrinsic dependency with the task label (\eg, in disease classification, the disease group has a different age range than the control group), such that the task label $\y$ could potentially become a moderator \cite{baron1986moderator} that affects the strength of dependency between the features and protected variables. In this situation, the goal of fair representation learning is to remove the direct statistical dependency between $\F$ and $\b$ while tolerating the indirect association induced by the task. Therefore, our adversarial training aims to ensure mean independence between $\F$ and $\b$ conditioned on the task label $E[\F|\b,\y]=E[\F|\y]\mbox{, }E[\b|\F,\y]=E[\b|\y].$ In practice, we train the adversarial loss within one or each of the $M$ classes, depending on the specific task. This alleviates the `competing equilibrium' issue in common fair machine learning methods \cite{xie2017controllable}, where the aim is to achieve full independence w.r.t $\b$ while accurately predict $\y$, an impossible task.



\section{Experiments}
We evaluate our method on two different scenarios. 
We compare BR-Net with several baseline approaches, and evaluate how our learned representations are invariant to the protected variables. 

\vspace{5pt}\noindent\textbf{Baseline Methods}. In line with the implementation of our approach, the baselines for all three experiments are 1) Vanilla: a vanilla CNN with an architecture exactly the same as BR-net without the bias prediction sub-network and hence the adversarial loss; and 2) Multi-task: a single $\FE$ followed by two separate predictors for predicting $\b_\kappa$ and $\y$, respectively \cite{lu2017fully}. The third type of approaches used for comparison are other unbiased representation learning methods. \textbf{Note that most existing works for ``fair deep learning'' are only designed for binary or categorical bias variables}. Therefore, in the synthetic and brain MRI experiments where the protected variable is continuous, we compare with two applicable scenarios originally proposed in the logistic regression setting: 1) \cite{sadeghi2019global} uses the MSE between the predicted and true bias as the adversarial loss; 2) \cite{zafar17a} aims to minimize the magnitude of correlation between $\b_\kappa$ and the logit of $\y$, which in our case is achieved by adding the correlation magnitude to the loss function. For the Gender Shades PPB experiment, the protected variable is categorical. We then further compare with \cite{kim2019}, which uses conditional entropy as the adversarial loss to minimize the mutual information between bias and features. Note, entropy-based \cite{springenberg2015unsupervised,roy2019} and mutual-information-based methods \cite{song2019,bertran19a,moyer2018} are widely used in fair representation learning to handle discrete bias.

\vspace{5pt}\noindent\textbf{Metrics for Accuracy and Statistical Independence}. For the MRI and GS-PPB experiments, we measure prediction accuracy of each method by recording the balanced accuracy (bAcc), $F_1$-score, and AUC from a 5-fold cross-validation. In addition, we track the statistical dependency between the protected variable and features during the training process by applying the model to the entire dataset. We then compute the squared distance correlation ($dcor^2$) \cite{szekely2007measuring} and mutual information (MI) between the learned features at each iteration and the ground-truth protected variable. Note that the computation of $dcor^2$ and MI does not involve the bias predictor ($\BP$), thereby enabling a unified comparison between adversarial methods and the non-adversarial ones. Unlike Pearson correlation, $dcor^2=0$ or MI$=0$ imply full statistical independence with respect to the features in the high dimensional space. Lastly, the discrete protected variable in GS-PPB experiment allows us to record another independence metric called the Equality of Opportunity (EO). EO measures the average gap in true positive rates w.r.t. different values of the protected variable.

\begin{figure*}[!t]
    \centering
    
    \includegraphics[width=0.95\textwidth]{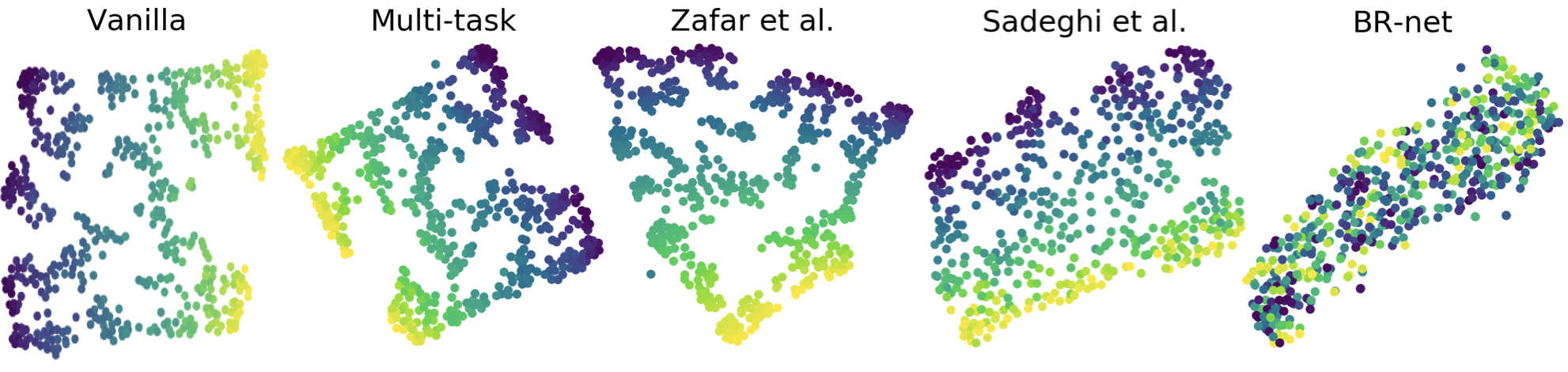}
    \caption{tSNE projection of the learned features for different methods. Color indicates the value of $\sigma_B$.}
    \label{fig:synth}
\end{figure*}

\subsection{Synthetic Experiments} 
We generate a synthetic dataset comprised of two groups of data, each containing 512 images of resolution $32\times32$ pixels. Each image is generated by 4 Gaussians (see Fig.~\ref{fig:synthetic_setup}a), the magnitude of which is controlled by  $\sigma_A$ and $\sigma_B$. For each image from Group 1, we sample $\sigma_A$  and $\sigma_B$ from a uniform distribution $\mathcal{U}(1,4)$ while we generate images of Group 2 with stronger intensities by sampling from $\mathcal{U}(3,6)$. Gaussian noise is added to the images with standard deviation $0.01$. Now we assume the difference in $\sigma_A$ between the two groups is associated with the true discriminative cues that should be learned by a classifier, whereas $\sigma_B$ is a protected variable. In other words, an unbiased model should predict the group label purely based on the two diagonal Gaussians and not dependent on the two off-diagonal ones. To show that the BR-Net can result in such models by controlling for $\sigma_B$, we train it on the whole dataset of 1,024 images with binary labels and $\sigma_B$ values.

For simplicity, we construct $\FE$ with 3 stacks of $2\times2$ convolution/ReLU/max-pooling layers to produce 32 features. Both the $\BP$ and $\CC$ networks have one hidden layer of dimension 16 with $tanh$ as the non-linear activation function. After training, the vanilla and multi-task models achieve close to 95\% training accuracy, and the other 3 methods  close to 90\%. Note that the theoretically maximum training accuracy is $90\%$ due to the overlapping sampling range of $\sigma_A$ between the two groups, indicating that the vanilla and multi-task models additionally rely on the protected variable $\sigma_B$ for predicting the group label, an undesired behavior. Further, Fig. \ref{fig:synthetic_setup}b shows that our method can optimally remove the statistical association w.r.t. $\sigma_B$ as $dcor^2$ drops dramatically with training iterations. The MSE-based adversarial loss yields unstable $dcor^2$ measures, and \cite{zafar17a} suboptimally removes the bias in the features (green curve Fig. \ref{fig:synthetic_setup}b). Finally, the above results are further supported by the 2D t-SNE \cite{maaten2008visualizing} projection of the learned features as shown in Fig.~\ref{fig:synth}. BR-net results in a feature space with no apparent bias, whereas features derived by other methods form a clear correlation with $\sigma_B$. This confirms the unbiased representation learned by BR-Net.

\subsection{HIV Diagnosis Based on MRIs}

\begin{figure}[!t]
    \centering
    \includegraphics[width=0.9\linewidth,trim=0 0 750 0,clip]{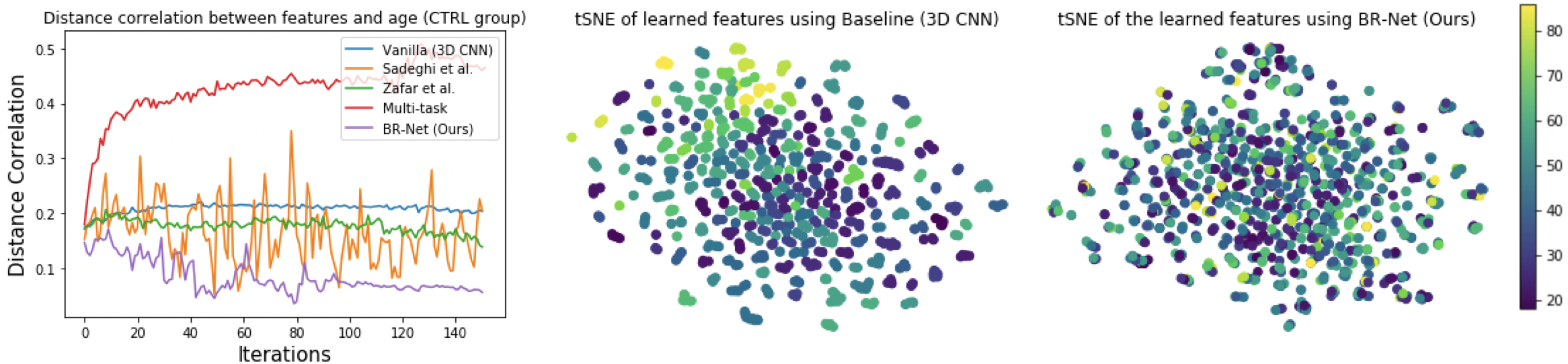}
    \caption{Statistical dependence between the learned features and age for the CTRL cohort in the HIV experiment, which is quantitatively measured by $dcor^2$.}
    \label{fig:hivcorr}
\end{figure}

\begin{table*}
\begin{minipage}{0.4\linewidth}
    \centering
    \captionof{table}{Classification accuracy of HIV diagnosis prediction and statistical dependency of learned features w.r.t. age. Best result in each column is typeset in bold and the second best is underlined. }
    {\small
    \setlength{\tabcolsep}{2.8pt}
    \begin{tabular}{l|ccc|cc}
    \toprule
       \textbf{Method}  &  \textbf{bAcc} & $\textbf{F}_1$ & \textbf{AUC} & $\textbf{dcor}^2$ & \textbf{MI}\\
       \hline
       Vanilla & 71.6 & 0.68 & 80.8 & 0.21 & 0.07 \\
       Multi-task & \textbf{74.2} & \underline{0.66} & \textbf{82.5} & 0.47 & 1.31 \\
       Sadeghi \etal~\cite{sadeghi2019global} & 64.8 & 0.58 & 75.2 & 0.22 &  0.06\\
       Zafar \etal~\cite{zafar17a} & \underline{73.2} & 0.65 & 80.8 & \underline{0.15} & \underline{0.04} \\
       BR-Net (Ours)&  \textbf{74.2} & \textbf{0.74} & \underline{80.9} & \textbf{0.05} & \textbf{7e-4}\\
       \bottomrule
    \end{tabular}
    }
    \label{tab:results}
\end{minipage}~
\begin{minipage}{0.59\linewidth}
    \centering
    \begin{tikzpicture}
        \draw (0, 0) node[inner sep=0] {
        \includegraphics[width=0.5\linewidth,trim=26 0 40 20,clip]{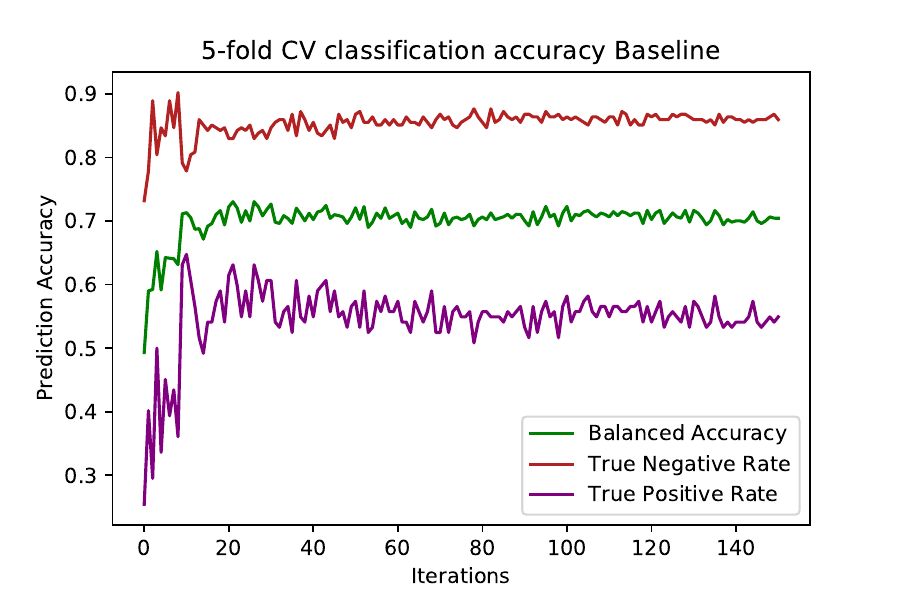}
        \includegraphics[width=0.5\linewidth,trim=26 0 40 20,clip]{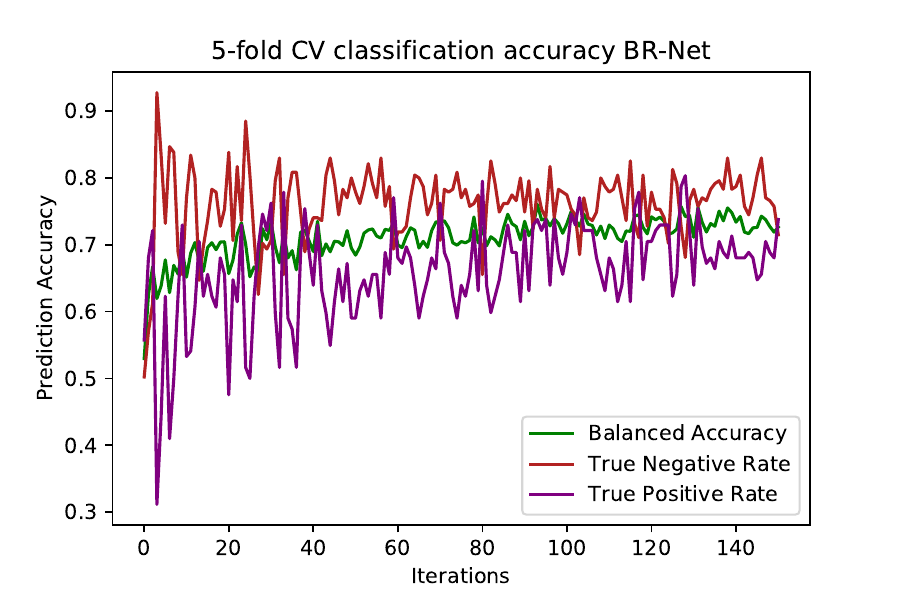}
        };
        \draw (-2.5, -.8) node {\scriptsize \textbf{(a)}};
        \draw (2.5, -.8) node {\scriptsize \textbf{(b)}};
    \end{tikzpicture}
    \vspace{-10pt}\captionof{figure}{Accuracy, TNR, and TPR of the HIV experiment, as a function of the \# of iterations for (a) 3D CNN baseline, (b) BR-Net. Our method is robust against the imbalanced age distribution between HIV and CTRL.}
    \label{fig:hiv_accplot}
\end{minipage}
\end{table*}


Neuroimaging studies increasingly rely on machine learning models to identify differences in brain images between cohorts. Our first task aims at classifying HIV patients \vs~control subjects (CTRL) based on brain MRIs to help understanding the impact of HIV on the brain. The study cohort includes 223 CTRLs and 122 HIV patients who are seropositive for the HIV-infection with CD4 count $>100~\frac{\text{cells}}{\mu L}$ (average: 303.0). 
Since the HIV subjects are significantly older in age than the CTRLs (CTRL: $45\pm17$, HIV: $51\pm8.3$, $p<.001$) in this study, normal aging becomes a potential task bias; prediction of diagnosis labels may be dependent on subjects' age instead of true HIV markers. 

The T1-weighted MRIs are all skull stripped, affinely registered to a common template, and resized into a $64\times64\times64$ volume. For each run of the 5-fold cross-validation, the training folds are augmented by random shifting (within one-voxel distance), rotation (within one degree) in all 3 directions, and left-right flipping based on the assumption that HIV infection affects the brain bilaterally \cite{adeli2018chained}. The data augmentation results in a balanced training set of 1024 CTRLs and 1024 HIVs. As the flipping removes left-right orientation, the ConvNet is built on half of the 3D volume containing one hemisphere. The representation extractor $\FE$ has 4 stacks of $2\times2\times2$ 3D convolution/ReLu/batch-normalization/max-pooling layers yielding 4096 intermediate features. Both $\BP$ and $\CC$ have one hidden layer of dimension 128 with $tanh$ as the activation function. As discussed, the task bias should be handled within individual groups rather the whole dataset. Motivated by recent medical studies \cite{rao2017predictive,adeli2018chained}, we perform the adversarial training with respect to the protected variable of age only on the CTRLs because HIV subjects may exhibit irregular aging. Extension analysis of this conditional modeling applied to medical applications can be found at \cite{zhaoadeli2020cf-net}.


Table \ref{tab:results} shows the diagnosis prediction accuracy of BR-Net in comparison with baseline methods. 
BR-Net results in the most accurate prediction in terms of balanced accuracy (bAcc) and $F_1$-score, while it also learns the least biased features in terms of $dcor^2$ and MI. Although prior works \cite{xie2017controllable} suggested that improving fairness of the model may reduce prediction accuracy due to the ``competing equilibrium'' between $\CC$ and $\BP$, our results on this relatively small dataset indicate that naive classifiers may easily overfit to the aging (bias) effect and therefore result in lower accuracy in HIV classification. On the other hand, BR-Net alleviates the ``competing equilibrium'' issue in the task bias by pursuing conditional independence (CTRL) between features and age. While the multi-task model produces a higher AUC, it is also the most biased model as it simultaneously leverages both age and HIV cues for prediction. This result is also supported by Fig.~\ref{fig:hivcorr}, where the distance correlation for BR-Net decreases with the adversarial training and increases for Multi-task. The MSE-based adversarial loss \cite{sadeghi2019global} yields unstable $dcor^2$ measures potentially due to the ill-posed optimization of maximizing $\ell_2$ distance. Moreover, minimizing the statistical association between the bias and predicted label $\y$ \cite{zafar17a} does not necessarily lead to unbiased features (green curve Fig. \ref{fig:hivcorr}).
In addition, we record the true positive and true negative rate of BR-net for each training iteration. As shown in Fig.~\ref{fig:hiv_accplot}, the baseline tends to predict most subjects as CTRLs (high true negative rate). This is potentially caused by the CTRL group having a wider age distribution, so an age-dependent predictor would bias the prediction towards CTRL. When controlling for age, BR-Net reliably results in balanced true positive and true negative rates.

\begin{table*}[!t]
    \centering
    \caption{Average results over five runs of 5-fold cross-validation (accuracy and statistical independence metrics) on GS-PPB. Best results are typeset in bold and second best are underlined.} 
    { \small
    \setlength{\tabcolsep}{3pt}
    \begin{tabular}{l cccccc c cccccc}
    \toprule
        ~ & \multicolumn{6}{c}{\textbf{VGG16 Backbone}} &~& \multicolumn{6}{c}{\textbf{ResNet50 Backbone}} \\
        \cline{2-7} \cline{9-14} 
        \textbf{Method} & \textbf{bAcc (\%)} & $\textbf{F}_1$  \textbf{(\%)}& \textbf{AUC (\%)} & $dcor^2$ & \textbf{MI} & \textbf{EO\%} &~& \textbf{bAcc (\%)} & $\textbf{F}_1$  \textbf{(\%)}& \textbf{AUC (\%)} & $dcor^2$ & \textbf{MI} & \textbf{EO\%}\\
        \hline
        Vanilla & 94.1$\pm$0.2 & 93.5$\pm$0.3 & 98.9$\pm$0.1 & 0.17 & 0.40 & 4.29&~&  90.7$\pm$0.7 & 89.8$\pm$0.7 & 97.8$\pm$0.1 & 0.29 & 0.60 & 11.2 \\
        {\scriptsize Kim \etal~\cite{kim2019}} & \underline{95.8$\pm$0.5} & \underline{95.7$\pm$0.5} & \underline{99.2$\pm$0.2} & 0.32  & \underline{0.28} & 4.12 &~& 91.4$\pm$0.9 & 91.0$\pm$0.9 & 96.6$\pm$0.7 & \textbf{0.18}&\underline{0.55} & \underline{3.86}\\
        {\scriptsize Zafar \etal~\cite{zafar17a}} &  94.3$\pm$0.4  &  93.7$\pm$0.5  & 99.0$\pm$0.1 & \underline{0.19} & 0.43 & \underline{4.11} &~& \textbf{94.2$\pm$0.4} & \textbf{93.6$\pm$0.4} & \textbf{98.7$\pm$0.1}& 0.29 & 0.60 & 4.68 \\
        Multi-Task & 94.0$\pm$0.3 & 93.4$\pm$0.3 & 98.9$\pm$0.1 & 0.28 & 0.42 & 4.45 &~& 94.0$\pm$0.3 & \underline{93.4$\pm$0.3} & \underline{98.6$\pm$0.3}  &0.29 & 0.63& 4.15\\
        BR-Net & \textbf{96.3$\pm$0.6} & \textbf{96.0$\pm$0.7} & \textbf{99.4$\pm$0.2} & \textbf{0.12} & \textbf{0.13} & \textbf{2.02} &~& \underline{94.1}$\pm$0.2 &  \textbf{93.6$\pm$0.2} & \underline{98.6$\pm$0.1} & \underline{0.23}& \textbf{0.49} & \textbf{2.87}\\
    \bottomrule
    \end{tabular}
    }
    \label{tab:gs_results}
\end{table*}

\begin{figure*}[!t]
    \centering
    \includegraphics[width=0.98\textwidth]{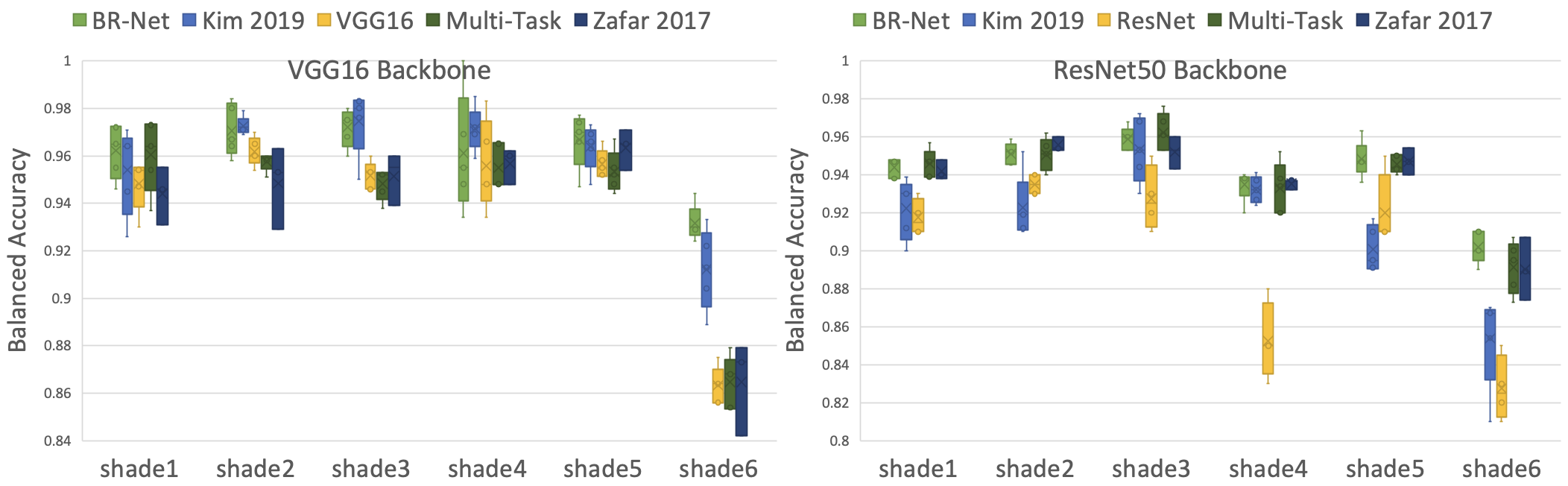}
    \caption{Accuracy of gender prediction from face images across all shades (1 to 6) of the GS-PPB dataset with two backbones, (left) VGG16 and (right) ResNet50. BR-Net consistently results in more accurate predictions in all 6 shade categories.}
    \label{fig:gs_diagram}
\end{figure*}
\begin{figure*}[!t]
    \centering
    \includegraphics[width=\textwidth]{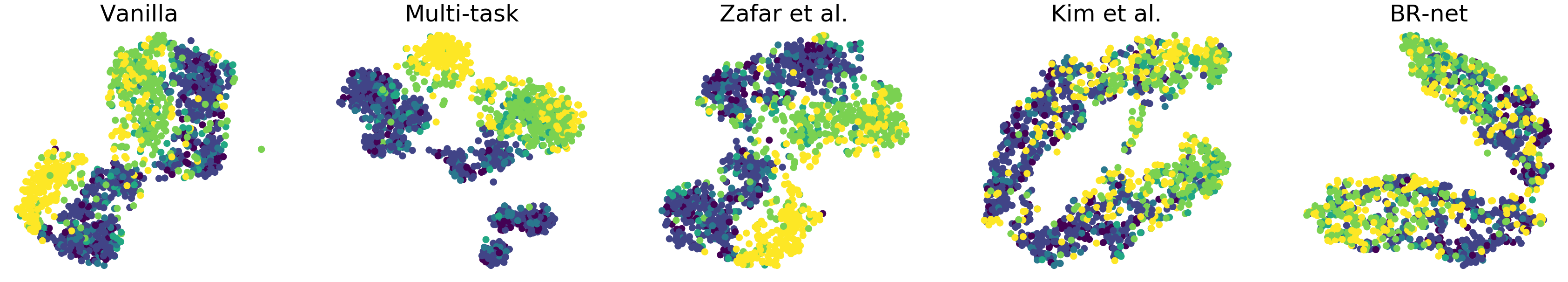}
    \caption{Learned representations by different methods. Color encodes the 6 categories of skin shade.}
    \label{fig:tsne_gs}
\end{figure*}

\subsection{Gender Prediction Using the GS-PPB Dataset}

The last experiment is on gender prediction from face images in the Gender Shades  Pilot Parliaments Benchmark (GS-PPB) dataset \cite{buolamwini2018gender}. This dataset contains 1,253 facial images of 561 female and 692 male subjects. The face shade is quantified by the Fitzpatrick six-point labeling system and is categorised from type 1 (lighter) to type 6 (darker). This quantization was used by dermatologists for skin classification and determining risk for skin cancer \cite{buolamwini2018gender}. To ensure prediction is purely based on facial areas, we first perform face detection and crop the images \cite{Adam2018}. 
To train our models on this dataset, we use backbones VGG16 \cite{simonyan2015very} and ResNet50 \cite{he2015deep}  pre-trained on ImageNet \cite{deng2009imagenet}. We fine-tune each model on GS-PPB dataset to predict the gender of subjects based on their face images using fair 5-fold cross-validation. The ImageNet dataset for pre-training the models has fewer cases of humans with darker faces \cite{yang2019towards}, and hence the resulting models have an underlying dataset bias to the shade.

BR-Net counts the variable `shade' as an ordinal and categorical protected variable. As discussed earlier, besides the baseline models in the HIV experiment, we additionally compare with a fair representation learning method, \cite{kim2019}, based on mutual information minimization. Note that this method is designed to handle discrete protected variables, therefore not applicable in previous experiments. We exclude \cite{sadeghi2019global} as the adversarial MSE-loss results in large oscillation in prediction results. Table \ref{tab:gs_results} shows the prediction results across five runs of 5-fold cross-validation and the independence metrics derived by training on the entire dataset. Fig.~\ref{fig:gs_diagram} plots the accuracy for each individual `shade' category. In terms of bAcc, BR-Net results in more accurate gender prediction than all baseline methods except that it it slightly worse than \cite{zafar17a} with ResNet50 backbone. However, features learned by \cite{zafar17a} are more biased towards skin shade. In most cases our method produces less biased features than \cite{kim2019}, a method designed to explicitly optimize full statistical independence between variables. In practice, removing mean dependency by adversarial training is potentially a better surrogate for removing statistical dependency between high-dimensional features and bias.

BR-Net produces similar accuracy across all `shade' categories. Prediction made by other methods, however, is more dependent on the protected variable by showing inconsistent recognition capabilities for different `shade' categories and failing significantly on darker faces. This bias is confirmed by the t-SNE projection of the feature spaces (see Fig.~\ref{fig:tsne_gs}) learned by the baseline methods; they all form clearer association with the bias variable than BR-Net. To gain more insight, we visualize the saliency maps derived for the baseline and BR-Net. For this purpose, we use a similar technique as in \cite{simonyan2014deep} to extract the pixels in the original image space highlighting the areas that are discriminative for the gender labels. Generating such saliency maps for all inputs, we visualize the average map for each individual `shade' category (Fig.~\ref{fig:gs_viz}). The value on each pixel corresponds to the attention from the network to that pixel within the classification process. Compared to the baseline, BR-Net focuses more on specific face regions and results in more stable patterns across all `shade' categories.

\section{Conclusion} 
Machine learning models are acceding to everyday lives from policy making to crucial medical applications. Failure to account for the underlying bias in datasets and tasks can lead to spurious associations and erroneous decisions. We proposed a method based on adversarial training strategies by encouraging vanished correlation to learn features for the prediction task while being unbiased to the protected variables in the study. We evaluated our bias-resilient neural network (BR-Net) on synthetic, medical diagnosis, and gender classification datasets. In all experiments, BR-Net resulted in representations that were invariant to the protected variable while obtaining comparable (and sometime better) classification accuracy. 

\noindent\textbf{Acknowledgement.} This work was supported in part by NIH Grants AA017347, AA010723, MH113406, and AA021697, and by Stanford HAI AWS Cloud Credit.

{\small
\bibliographystyle{ieee_fullname}
\bibliography{References}

\begin{thebibliography}{10}\itemsep=-1pt

\bibitem{adeli2018chained}
Ehsan Adeli, Dongjin Kwon, Qingyu Zhao, Adolf Pfefferbaum, Natalie~M Zahr,
  Edith~V Sullivan, and Kilian~M Pohl.
\newblock Chained regularization for identifying brain patterns specific to
  {HIV} infection.
\newblock {\em NeuroImage}, 183:425--437, 2018.

\bibitem{akuzawa2019adversarial}
Kei Akuzawa, Yusuke Iwasawa, and Yutaka Matsuo.
\newblock Adversarial invariant feature learning with accuracy constraint for
  domain generalization.
\newblock {\em arXiv preprint arXiv:1904.12543}, 2019.

\bibitem{barocas2017fairness}
Solon Barocas, Moritz Hardt, and Arvind Narayanan.
\newblock Fairness in machine learning.
\newblock {\em Neural Information Processing Systems Tutorial}, 2017.

\bibitem{barocas2016big}
Solon Barocas and Andrew~D Selbst.
\newblock Big data's disparate impact.
\newblock {\em Calif. L. Rev.}, 104:671, 2016.

\bibitem{baron1986moderator}
Reuben~M Baron and David~A Kenny.
\newblock The moderator--mediator variable distinction in social psychological
  research: Conceptual, strategic, and statistical considerations.
\newblock {\em Journal of personality and social psychology}, 51(6):1173, 1986.

\bibitem{bechavod2017learning}
Yahav Bechavod and Katrina Ligett.
\newblock Learning fair classifiers: A regularization-inspired approach.
\newblock {\em arXiv preprint arXiv:1707.00044}, pages 1733--1782, 2017.

\bibitem{bertran19a}
Martin Bertran, Natalia Martinez, Afroditi Papadaki, Qiang Qiu, Miguel
  Rodrigues, Galen Reeves, and Guillermo Sapiro.
\newblock Adversarially learned representations for information obfuscation and
  inference.
\newblock In Kamalika Chaudhuri and Ruslan Salakhutdinov, editors, {\em
  Proceedings of the 36th International Conference on Machine Learning},
  volume~97, pages 614--623, 2019.

\bibitem{beutel2017data}
Alex Beutel, Jilin Chen, Zhe Zhao, and Ed~H Chi.
\newblock Data decisions and theoretical implications when adversarially
  learning fair representations.
\newblock {\em arXiv preprint arXiv:1707.00075}, 2017.

\bibitem{bousmalis2017unsupervised}
Konstantinos Bousmalis, Nathan Silberman, David Dohan, Dumitru Erhan, and Dilip
  Krishnan.
\newblock Unsupervised pixel-level domain adaptation with generative
  adversarial networks.
\newblock In {\em Proceedings of the IEEE Conference on Computer Vision and
  Pattern Pecognition}, pages 3722--3731, 2017.

\bibitem{buolamwini2018gender}
Joy Buolamwini and Timnit Gebru.
\newblock Gender shades: Intersectional accuracy disparities in commercial
  gender classification.
\newblock In {\em Conference on fairness, accountability and transparency},
  pages 77--91, 2018.

\bibitem{celis2016fair}
L~Elisa Celis, Amit Deshpande, Tarun Kathuria, and Nisheeth~K Vishnoi.
\newblock How to be fair and diverse?
\newblock {\em arXiv preprint arXiv:1610.07183}, 2016.

\bibitem{choi2019can}
Jinwoo Choi, Chen Gao, Joseph~CE Messou, and Jia-Bin Huang.
\newblock Why can't i dance in the mall? learning to mitigate scene bias in
  action recognition.
\newblock In {\em Advances in Neural Information Processing Systems}, pages
  851--863, 2019.

\bibitem{cole2017increased}
James~H Cole, Jonathan Underwood, Matthan~WA Caan, Davide De~Francesco, Rosan~A
  van Zoest, Robert Leech, Ferdinand~WNM Wit, Peter Portegies, Gert~J Geurtsen,
  Ben~A Schmand, et~al.
\newblock Increased brain-predicted aging in treated {HIV} disease.
\newblock {\em Neurology}, 88(14):1349--1357, 2017.

\bibitem{creager2019flexibly}
Elliot Creager, David Madras, Joern-Henrik Jacobsen, Marissa Weis, Kevin
  Swersky, Toniann Pitassi, and Richard Zemel.
\newblock Flexibly fair representation learning by disentanglement.
\newblock In Kamalika Chaudhuri and Ruslan Salakhutdinov, editors, {\em
  Proceedings of the 36th International Conference on Machine Learning},
  volume~97, pages 1436--1445, Long Beach, California, USA, 09--15 Jun 2019.
  PMLR.

\bibitem{deng2009imagenet}
Jia Deng, Wei Dong, Richard Socher, Li-Jia Li, Kai Li, and Li Fei-Fei.
\newblock Imagenet: A large-scale hierarchical image database.
\newblock In {\em 2009 IEEE Conference on Computer Vision and Pattern
  Recognition}, pages 248--255. Ieee, 2009.

\bibitem{dobrowolska2019age}
Beata Dobrowolska, Bernadeta J{\k{e}}drzejkiewicz, Anna Pilewska-Kozak, Danuta
  Zarzycka, Barbara {\'S}lusarska, Alina Deluga, Aneta Ko{\'s}cio{\l}ek, and
  Alvisa Palese.
\newblock Age discrimination in healthcare institutions perceived by seniors
  and students.
\newblock {\em Nursing ethics}, 26(2):443--459, 2019.

\bibitem{Edwards16}
Harrison Edwards and Amos Storkey.
\newblock Censoring representations with an adversary.
\newblock In {\em International Conference on Learning Representations}, 2016.

\bibitem{eichler1992gender}
Margrit Eichler, Anna~Lisa Reisman, and Elaine~Manace Borins.
\newblock Gender bias in medical research.
\newblock {\em Women \& Therapy}, 12(4):61--70, 1992.

\bibitem{elazargoldberg2018adversarial}
Yanai Elazar and Yoav Goldberg.
\newblock Adversarial removal of demographic attributes from text data.
\newblock In {\em Proceedings of the 2018 Conference on Empirical Methods in
  Natural Language Processing}, pages 11--21, Brussels, Belgium, Oct.-Nov.
  2018. Association for Computational Linguistics.

\bibitem{ganin2016domain}
Yaroslav Ganin, Evgeniya Ustinova, Hana Ajakan, Pascal Germain, Hugo
  Larochelle, Fran{\c{c}}ois Laviolette, Mario Marchand, and Victor Lempitsky.
\newblock Domain-adversarial training of neural networks.
\newblock {\em The Journal of Machine Learning Research}, 17(1):2096--2030,
  2016.

\bibitem{Adam2018}
Adam Geitgey.
\newblock Face recognition, \url{https://pypi.org/project/face_recognition/},
  2018.

\bibitem{goodfellow2014generative}
Ian Goodfellow, Jean Pouget-Abadie, Mehdi Mirza, Bing Xu, David Warde-Farley,
  Sherjil Ozair, Aaron Courville, and Yoshua Bengio.
\newblock Generative adversarial nets.
\newblock In {\em Advances in neural information processing systems}, pages
  2672--2680, 2014.

\bibitem{hashimoto2018fairness}
Tatsunori~B Hashimoto, Megha Srivastava, Hongseok Namkoong, and Percy Liang.
\newblock Fairness without demographics in repeated loss minimization.
\newblock In {\em Proceedings of the International Conference on Machine
  Learning}, 2018.

\bibitem{he2015deep}
K He, X Zhang, S Ren, and J Sun.
\newblock Deep residual learning for image recognition.
\newblock In {\em IEEE Conference on Computer Vision and Pattern Recognition
  (CVPR)}, volume~5, pages 770--778, 2015.

\bibitem{hoffman2017cycada}
Judy Hoffman, Eric Tzeng, Taesung Park, Jun-Yan Zhu, Phillip Isola, Kate
  Saenko, Alexei~A Efros, and Trevor Darrell.
\newblock Cycada: Cycle-consistent adversarial domain adaptation.
\newblock {\em arXiv preprint arXiv:1711.03213}, 2017.

\bibitem{huang2012biased}
Fei Huang and Alexander Yates.
\newblock Biased representation learning for domain adaptation.
\newblock In {\em Proceedings of the 2012 Joint Conference on Empirical Methods
  in Natural Language Processing and Computational Natural Language Learning},
  pages 1313--1323. Association for Computational Linguistics, 2012.

\bibitem{khademi2019algorithmic}
Aria Khademi and Vasant Honavar.
\newblock Algorithmic bias in recidivism prediction: A causal perspective.
\newblock {\em Proceedings of the AAAI Conference on Artificial Intelligence},
  pages 13839--13840, Apr. 2020.

\bibitem{khosla2012undoing}
Aditya Khosla, Tinghui Zhou, Tomasz Malisiewicz, Alexei~A Efros, and Antonio
  Torralba.
\newblock Undoing the damage of dataset bias.
\newblock In {\em European Conference on Computer Vision}, pages 158--171.
  Springer, 2012.

\bibitem{Dhruv2019}
Dhruv Khullar.
\newblock {A.I.} could worsen health disparities.
\newblock {\em New York Times}, 2019.

\bibitem{kim2019}
Byungju Kim, Hyunwoo Kim, Kyungsu Kim, Sungjin Kim, and Junmo Kim.
\newblock Learning not to learn: Training deep neural networks with biased
  data.
\newblock In {\em Proceedings of the IEEE Conference on Computer Vision and
  Pattern Recognition}, pages 9012--9020, 2019.

\bibitem{lee2019algorithmic}
Nicol~Turner Lee, Paul Resnick, and Genie Barton.
\newblock Algorithmic bias detection and mitigation: Best practices and
  policies to reduce consumer harms.
\newblock {\em Center for Technology Innovation, Brookings. Tillg{\"a}nglig
  online: https://www. brookings.
  edu/research/algorithmic-bias-detection-and-mitigation-bestpractices-and-policies-to-reduce-consumer-harms/\#
  footnote-7 (2019-10-01)}, 2019.

\bibitem{li2019repair}
Yi Li and Nuno Vasconcelos.
\newblock {REPAIR}: Removing representation bias by dataset resampling.
\newblock In {\em Proceedings of the IEEE Conference on Computer Vision and
  Pattern Recognition}, pages 9572--9581, 2019.

\bibitem{liu2018delayed}
Lydia~T Liu, Sarah Dean, Esther Rolf, Max Simchowitz, and Moritz Hardt.
\newblock Delayed impact of fair machine learning.
\newblock In {\em Proceedings of the International Conference on Machine
  Learning}, 2018.

\bibitem{lu2017fully}
Yongxi Lu, Abhishek Kumar, Shuangfei Zhai, Yu Cheng, Tara Javidi, and Rogerio
  Feris.
\newblock Fully-adaptive feature sharing in multi-task networks with
  applications in person attribute classification.
\newblock In {\em Proceedings of the IEEE Conference on Computer Vision and
  Pattern Recognition}, pages 5334--5343, 2017.

\bibitem{maaten2008visualizing}
Laurens van~der Maaten and Geoffrey Hinton.
\newblock Visualizing data using {t-SNE}.
\newblock {\em Journal of Machine Learning Research}, 9:2579--2605, 2008.

\bibitem{madras2018learning}
David Madras, Elliot Creager, Toniann Pitassi, and Richard Zemel.
\newblock Learning adversarially fair and transferable representations.
\newblock {\em arXiv preprint arXiv:1802.06309}, 2018.

\bibitem{Miller2015}
Clair Miller.
\newblock Algorithms and bias.
\newblock {\em New York Times}, 2015.

\bibitem{moyer2018}
Daniel Moyer, Shuyang Gao, Rob Brekelmans, Aram Galstyan, and Greg Ver~Steeg.
\newblock Invariant representations without adversarial training.
\newblock In S. Bengio, H. Wallach, H. Larochelle, K. Grauman, N. Cesa-Bianchi,
  and R. Garnett, editors, {\em Advances in Neural Information Processing
  Systems 31}, pages 9084--9093. 2018.

\bibitem{pourhoseingholi2012control}
Mohamad~Amin Pourhoseingholi, Ahmad~Reza Baghestani, and Mohsen Vahedi.
\newblock How to control confounding effects by statistical analysis.
\newblock {\em Gastroenterol Hepatol Bed Bench}, 5(2):79, 2012.

\bibitem{ranzato2007unsupervised}
Marc'aurelio Ranzato, Fu~Jie Huang, Y-Lan Boureau, and Yann LeCun.
\newblock Unsupervised learning of invariant feature hierarchies with
  applications to object recognition.
\newblock In {\em 2007 IEEE conference on computer vision and pattern
  recognition}, pages 1--8. IEEE, 2007.

\bibitem{rao2017predictive}
Anil Rao et~al.
\newblock Predictive modelling using neuroimaging data in the presence of
  confounds.
\newblock {\em NeuroImage}, 150:23--49, 2017.

\bibitem{roy2019}
Proteek~Chandan Roy and Vishnu~Naresh Boddeti.
\newblock Mitigating information leakage in image representations: {A} maximum
  entropy approach.
\newblock {\em CoRR}, abs/1904.05514, 2019.

\bibitem{sadeghi2019global}
Bashir Sadeghi, Runyi Yu, and Vishnu Boddeti.
\newblock On the global optima of kernelized adversarial representation
  learning.
\newblock In {\em Proceedings of the IEEE International Conference on Computer
  Vision}, pages 7971--7979, 2019.

\bibitem{salimi2019interventional}
Babak Salimi, Luke Rodriguez, Bill Howe, and Dan Suciu.
\newblock Interventional fairness: Causal database repair for algorithmic
  fairness.
\newblock In {\em Proceedings of the 2019 International Conference on
  Management of Data}, pages 793--810, 2019.

\bibitem{simonyan2014deep}
Karen Simonyan, Andrea Vedaldi, and Andrew Zisserman.
\newblock Deep inside convolutional networks: Visualising image classification
  models and saliency maps.
\newblock In {\em International Conference on Learning Representations}, 2014.

\bibitem{simonyan2015very}
Karen Simonyan and Andrew Zisserman.
\newblock Very deep convolutional networks for large-scale image recognition.
\newblock In {\em International Conference on Learning Representations}, 2015.

\bibitem{song2019}
Jiaming Song, Pratyusha Kalluri, Aditya Grover, Shengjia Zhao, and Stefano
  Ermon.
\newblock Learning controllable fair representations.
\newblock In Kamalika Chaudhuri and Masashi Sugiyama, editors, {\em Proceedings
  of Machine Learning Research}, volume~89, pages 2164--2173. PMLR, 2019.

\bibitem{springenberg2015unsupervised}
Jost~Tobias Springenberg.
\newblock Unsupervised and semi-supervised learning with categorical generative
  adversarial networks.
\newblock {\em arXiv preprint arXiv:1511.06390}, 2015.

\bibitem{szekely2007measuring}
G{\'a}bor~J Sz{\'e}kely, Maria~L Rizzo, Nail~K Bakirov, et~al.
\newblock Measuring and testing dependence by correlation of distances.
\newblock {\em The annals of statistics}, 35(6):2769--2794, 2007.

\bibitem{tommasi2017deeper}
Tatiana Tommasi, Novi Patricia, Barbara Caputo, and Tinne Tuytelaars.
\newblock A deeper look at dataset bias.
\newblock In {\em Domain adaptation in computer vision applications}, pages
  37--55. Springer, 2017.

\bibitem{tzeng2017adversarial}
Eric Tzeng, Judy Hoffman, Kate Saenko, and Trevor Darrell.
\newblock Adversarial discriminative domain adaptation.
\newblock In {\em Proceedings of the IEEE Conference on Computer Vision and
  Pattern Recognition}, pages 7167--7176, 2017.

\bibitem{wang2019balanced}
Tianlu Wang, Jieyu Zhao, Mark Yatskar, Kai-Wei Chang, and Vicente Ordonez.
\newblock Balanced datasets are not enough: Estimating and mitigating gender
  bias in deep image representations.
\newblock In {\em Proceedings of the IEEE International Conference on Computer
  Vision}, pages 5310--5319, 2019.

\bibitem{weinzaepfel2019mimetics}
Philippe Weinzaepfel and Gr{\'e}gory Rogez.
\newblock Mimetics: Towards understanding human actions out of context.
\newblock {\em arXiv preprint arXiv:1912.07249}, 2019.

\bibitem{wooldridge2010}
Jeffrey~M. Wooldridge.
\newblock {\em Econometric Analysis of Cross Section and Panel Data (2nd ed.)}.
\newblock The MIT Press, London, 2010.

\bibitem{xie2017controllable}
Qizhe Xie, Zihang Dai, Yulun Du, Eduard Hovy, and Graham Neubig.
\newblock Controllable invariance through adversarial feature learning.
\newblock In {\em Advances in Neural Information Processing Systems}, pages
  585--596, 2017.

\bibitem{yang2019towards}
Kaiyu Yang, Klint Qinami, Li Fei-Fei, Jia Deng, and Olga Russakovsky.
\newblock Towards fairer datasets: Filtering and balancing the distribution of
  the people subtree in the imagenet hierarchy.
\newblock {\em Image-Net}, 2019.

\bibitem{zafar17a}
Muhammad~Bilal Zafar, Isabel Valera, Manuel~Gomez Rogriguez, and Krishna~P.
  Gummadi.
\newblock {Fairness Constraints: Mechanisms for Fair Classification}.
\newblock In Aarti Singh and Jerry Zhu, editors, {\em Proceedings of the 20th
  International Conference on Artificial Intelligence and Statistics},
  volume~54, pages 962--970. PMLR, 2017.

\bibitem{zemel2013learning}
Rich Zemel, Yu Wu, Kevin Swersky, Toni Pitassi, and Cynthia Dwork.
\newblock Learning fair representations.
\newblock In {\em International Conference on Machine Learning}, pages
  325--333, 2013.

\bibitem{zhang2018mitigating}
Brian~Hu Zhang, Blake Lemoine, and Margaret Mitchell.
\newblock Mitigating unwanted biases with adversarial learning.
\newblock In {\em Proceedings of the 2018 AAAI/ACM Conference on AI, Ethics,
  and Society}, pages 335--340. ACM, 2018.

\bibitem{zhaoadeli2020cf-net}
Qingyu Zhao*, Ehsan Adeli*, and Kilian~M Pohl.
\newblock Training confounder-free deep learning models for medical
  applications.
\newblock {\em Nature Communications}, 2020 In Press.

\end{thebibliography}
}

\end{document}